\documentclass[11pt, a4paper, reqno, twoside]{scrartcl}

\usepackage[a4paper, hmargin={1in, 1in}, vmargin={1in,
  1in},headsep=10mm,headheight=5mm,footskip=30pt]{geometry}
 
\usepackage{graphicx, amsmath, amssymb,amsthm}
\usepackage{hyperref}
\usepackage{bm,color}
\usepackage{mathdots}

\usepackage{fancyhdr}
\usepackage{wrapfig}
\usepackage{multicol}
\usepackage[boxruled,commentsnumbered]{algorithm2e}

\definecolor{darkgreen}{rgb}{0,.4,.2}
\definecolor{darkblue}{rgb}{.1,.2,.6}
\definecolor{brightblue}{rgb}{0,0.6,0.8}
\hypersetup{
  colorlinks=true,
  linkcolor=darkblue,
  citecolor=darkgreen,
  filecolor=darkblue,
  urlcolor=darkblue
}

\addtokomafont{caption}{\small}
\addtokomafont{captionlabel}{\sffamily\bfseries}

\newtheoremstyle{style}
 {\topsep}              
 {\topsep}              
 {\itshape}              
 {}                         
 {\sffamily\bfseries}  
 {.}	                     
 { }                         
 {}                          
\theoremstyle{style}

\newtheorem*{rep@theorem}{\rep@title}
\newcommand{\newreptheorem}[2]{%
\newenvironment{rep#1}[1]{%
 \def\rep@title{#2 \ref{##1}}%
 \begin{rep@theorem}}%
 {\end{rep@theorem}}}

\newreptheorem{lemma}{Lemma'}
\newreptheorem{proposition}{Proposition'}
\newreptheorem{theorem}{Theorem'}

\newtheorem{definition}{Definition}
\newtheorem{theorem}[definition]{Theorem}
\newtheorem{proposition}[definition]{Proposition}
\newtheorem{lemma}[definition]{Lemma}

\newcommand{\conv}{\mathop{\rm conv}}
\newcommand{\cone}{\mathop{\rm cone}}
\newcommand{\intcone}{\,\buildrel{\circ}\over{\cone}\!}
\providecommand{\norm}[1]{\left\lVert#1\right\rVert}
\newcommand{\R}{\mathbb{R}}
\newcommand{\Simplex}{\triangle}
\newcommand{\FilledSimplex}{\blacktriangle}
\newcommand{\LOneBall}{\blacklozenge} 
\newcommand{\simpl}{{\!\scriptstyle{\vartriangle}}} 
\newcommand{\fSimpl}{{\!\scriptstyle{\blacktriangle}}}
\newcommand{\lOne}{{\diamond}} 

\newcommand{\id}{\mathbf{I}} 
\newcommand{\unit}{\mathbf{e}} 
\newcommand{\one}{\mathbf{1}} 
\newcommand{\zero}{\mathbf{0}}
\newcommand\SetOf[2]{\left\{#1\,\vphantom{#2}\right.\left|\vphantom{#1}\,#2\right\}}

\begin{document}
\pagestyle{fancy}

\title{An Equivalence between the Lasso \\and Support Vector Machines}

\author{\vspace{-2mm}\\\vspace{-1mm}
Martin Jaggi\footnote{%
A shorter version of this article has appeared in \cite{Jaggi:2013vta}.
}%
\\\vspace{-1mm}
{\small ETH Z{\"u}rich, Switzerland}\\\vspace{-1mm}
{\small \href{mailto:jaggi@inf.ethz.ch}{jaggi@inf.ethz.ch}}\vspace{-1em}}
\date{}

\maketitle

\begin{abstract}
\textbf{\textsf{Abstract.}}
We investigate the relation of two fundamental tools in machine learning and signal processing, that is the support vector machine (SVM) for classification, and the Lasso technique used in regression. We show that the resulting optimization problems are equivalent, in the following sense. Given any instance of an $\ell_2$-loss soft-margin (or hard-margin) SVM, we construct a Lasso instance having the same optimal solutions, and vice versa.

As a consequence, many existing optimization algorithms for both SVMs and Lasso can also be applied to the respective other problem instances.
Also, the equivalence allows for many known theoretical insights for SVM and Lasso to be translated between the two settings. One such implication gives a simple kernelized version of the Lasso, analogous to the kernels used in the SVM setting.
Another consequence is that the sparsity of a Lasso solution is equal to the number of support vectors for the corresponding SVM instance, and that one can use screening rules to prune the set of support vectors.
Furthermore, we can relate sublinear time algorithms for the two problems, and give a new such algorithm variant for the Lasso.
We also study the regularization paths for both methods.
\end{abstract}

\section{Introduction}

Linear classifiers and kernel methods, and in particular the support vector machine (SVM) \cite{cortes95softmargin}, are among the most popular standard tools for classification. 
On the other hand, $\ell_1$-regularized least squares regression, i.e. the Lasso estimator~\cite{Tibshirani:1996wb}, is one of the most widely used tools for robust regression and sparse estimation.
 
Along with the many successful practical applications of SVMs and the Lasso in various fields, there is a vast amount of existing literature\footnote{%
As of January 2014, Google Scholar returned $300'000$ publications containing the term \textsf{"Support Vector Machine"}, and over $20'000$ for \textsf{Lasso regression}.
} %
 on the two methods themselves, considering both theory and also algorithms for each of the two. 
However, the two research topics developed largely independently and were not much set into context with each other so far.

In this chapter, we attempt to better relate the two problems, with two main goals in mind.
On the algorithmic side, we show that 
many of the existing algorithms for each of the two problems can be set into comparison, and can be directly applied to the other respective problem.
As a particular example of this idea, we can apply the recent sublinear time SVM algorithm by \cite{Clarkson:2010ue} also to any Lasso problem, resulting in a new alternative sublinear time algorithm variant for the Lasso.

On the other hand, we can relate and transfer existing theoretical results between the  literature for SVMs and the Lasso.
In this spirit, a first example is the idea of the kernel trick. Originally employed for SVMs, this powerful concept has allowed for lifting most insights from the linear classifier setting also to the more general setting of non-linear classifiers, by using an implicit higher dimensional space.
Here, by using our equivalence, we propose a simple kernelized variant of the Lasso, being equivalent to the well-researched use of kernels in the SVM setting.

As another example, we can also transfer some insights in the other direction, from the Lasso to SVMs. The important datapoints, i.e. those that define the solution to the SVM classifier, are called the support vectors. Having a small set of support vectors is crucial for the practical performance of SVMs.
Using our equivalence, we see that the set of support vectors for a given SVM instance is exactly the same as the sparsity pattern of the corresponding Lasso solution.

Screening rules are a way of pre-processing the input data for a Lasso problem, in order to identify inactive variables. We show that screening rules can also be applied to SVMs, in order to eliminate potential support vectors beforehand, and thereby speeding up the training process by reducing the problem size.

Finally, we study the complexity of the solution paths of Lasso and SVMs, as the regularization parameter changes. We discuss path algorithms that apply to both problems, and also translate a result on the Lasso path complexity to show that a single SVM instance, depending on the scaling of the data, can have very different patterns of support vectors.

\paragraph{Support Vector Machines.}\index{Support Vector Machine}\index{SVM}
In this chapter, we focus on large margin linear classifiers\index{linear classifier}, and more precisely on those SVM variants whose dual optimization problem is of the form\index{SVM dual optimization problem}
\begin{equation}\label{eq:classifier}
  \min_{x\in\Simplex} \, \norm{Ax}^2 \ .
\end{equation}
Here the matrix $A\in\R^{d\times n}$ contains all $n$ datapoints as its columns, and $\Simplex$ is the unit simplex in $\R^n$, i.e. the set of probability vectors, that is the non-negative vectors whose entries sum up to one.
The formulation~(\ref{eq:classifier}) includes the commonly used soft-margin SVM with $\ell_2$-loss. It includes both the one or two classes variants, both with or without using a kernel, and both using a (regularized) offset term (allowing hyperplanes not passing through the origin) or no offset.
We will explain these variants and the large margin interpretation of this optimization problem in more detail in Section~\ref{sec:SVMproperties}.

\paragraph{Lasso.}\index{Lasso}\index{L1 regularized least squares regression@$\ell_1$-regularized least squares regression}
On the other hand, the Lasso~\cite{Tibshirani:1996wb}, is given by the quadratic program
\begin{equation}\label{eq:lasso}
  \min_{x\in\LOneBall} \, \norm{Ax-b}^2 \ .
\end{equation}
It is also known as the constrained variant of $\ell_1$-regularized least squares regression\index{least squares regression}. Here the right hand side $b$ is a fixed vector $b\in\R^d$, and $\LOneBall$ is the $\ell_1$-unit-ball in~$\R^n$. Note that the $1$-norm~$\norm{.}_1$ is the sum of the absolute values of the entries of a vector.
Sometimes in practice, one would like to have the constraint $\norm{x}_1\le r$ for some value $r>0$, instead of the simple unit-norm case $\norm{x}_1\le 1$. 
However, in that case it is enough to simply re-scale the input matrix~$A$ by a factor of $r$, in order to obtain our above formulation~(\ref{eq:lasso}) for any general Lasso problem.

In applications of the Lasso, it is important to distinguish two alternative interpretations of the data matrix~$A$, which defines the problem instance (\ref{eq:lasso}):
on one hand, in the setting of \emph{sparse regression}\index{sparse regression}, the matrix $A$ is usually called the dictionary matrix, with its columns $A_{:j}$ being the dictionary elements, and the goal being to approximate the single vector $b$ by a combination of few dictionary vectors.
On the other hand if the Lasso problem is interpreted as \emph{feature-selection}, then each row $A_{i:}$ of the matrix $A$ is interpreted as an input vector, and for each of those, the Lasso is approximating the response $b_i$ to input row $A_{i:}$.
The book \cite{Buhlmann:2011jp} gives a recent overview of Lasso-type methods and their broad applications.
While the penalized variant of the Lasso (meaning that the term $\norm{x}_1$ is added to the objective instead of imposed as a constraint) is also popular in applications, here we focus on the original constrained variant.

\paragraph{The Equivalence.}\index{equivalent optimization problems}
We will prove that the two problems (\ref{eq:classifier}) and (\ref{eq:lasso}) are indeed equivalent, in the following sense.
For any Lasso instance given by~$(A,b)$, we construct an equivalent (hard-margin) SVM instance, having the same optimal solution. This will be a simple reduction preserving all objective values. On the other hand, the task of finding an equivalent Lasso instance for a given SVM appears to be a harder problem. Here we show that there always exists such an equivalent Lasso instance, and furthermore, if we are given a weakly-separating vector for the SVM (formal definition to follow soon below), then we can explicitly construct the equivalent Lasso instance. This reduction also applies to the $\ell_2$-loss soft-margin SVM, where we show that a weakly-separating vector is trivial to obtain, making the reduction efficient.
The reduction does \emph{not} require that the SVM input data is separable.

Our shown equivalence is on the level of the (SVM or Lasso) \emph{training} formulations, we're not making any claims on the performance of the two different kind of methods on unseen \emph{test} data.

On the way to this goal, we will also explain the relation to the ``non-negative'' Lasso\index{non-negative regression}\index{non-negative Lasso} variant when the variable vector $x$ is required to lie in the simplex, i.e.
\begin{equation}\label{eq:posLasso}
  \min_{x\in\Simplex} \, \norm{Ax-b}^2 \ .
\end{equation}
It turns out the equivalence of the optimization problems~(\ref{eq:classifier}) and~(\ref{eq:posLasso}) is straightforward to see. Our main contribution is to explain the relation of these two optimization problems to the original Lasso problem~(\ref{eq:lasso}), and to study some of the implications of the equivalence.

%
%

\paragraph{Related Work.}
The early work of \cite{Girosi:1998dd} has already significantly deepened the joint understanding of kernel methods and the sparse coding setting of the Lasso. Despite its title, \cite{Girosi:1998dd} is not addressing SVM classifiers, but in fact the $\varepsilon$-insensitive loss variant of support vector regression\index{Support Vector Regression}\index{SVR} (SVR), which the author proves to be equivalent to a Lasso problem where~$\varepsilon$ then becomes the $\ell_1$-regularization.
Unfortunately, this reduction does not apply anymore when $\varepsilon=0$, which is the case of interest for standard hinge-loss SVR, and also for SVMs in the classification setting, which are the focus of our work here.

Another reduction has been known if the SVR insensitivity parameter~$\varepsilon$ is chosen close enough to one. In that case, \cite{Pontil:1998vs} has shown that the SVR problem can be reduced to SVM classification with standard hinge-loss. Unfortunately, this reduction does not apply to Lasso regression.

In a different line of research, \cite{Li:2005wr} have studied the relation of a dual variant of the Lasso to the primal of the so called \emph{potential SVM} originally proposed by \cite{Hochreiter:2004wv,Hochreiter:2006jc}
, which is not a classifier but a specialized method of feature selection.

In the application paper \cite{Ghosh:2005ft} in the area of computational biology, the authors already suggested to make use of the ``easier'' direction of our reduction, reducing the Lasso to a very particular SVM instance. The idea is to employ the standard trick of using two non-negative vectors to represent a point in the $\ell_1$-ball~\cite{Bloomfield:1983tn,Chen:1998hm}. Alternatively, this can also be interpreted as considering an SVM defined by all Lasso dictionary vectors together with their negatives ($2n$ many points). We formalize this interpretation more precisely in  Section~\ref{sec:LassoLeSVM}. The work of \cite{Ghosh:2005ft} does not address the SVM regularization parameter.

%
\paragraph{Notation.}
The following three sets of points will be central for our investigations. We denote the 
unit simplex\index{simplex}, the filled simplex as well as the $\ell_1$-unit-ball\index{L1 ball@$\ell_1$-ball} in $\R^n$ as follows.
\begin{eqnarray*}
\Simplex &:=& \textstyle\SetOf{x\in\R^n}{x\ge\zero,\ \sum_i x_i = 1} \ ,\\
\FilledSimplex &:=& \textstyle\SetOf{x\in\R^n}{x\ge\zero,\ \sum_i x_i \le 1} \ ,\\
\LOneBall &:=& \textstyle\SetOf{x\in\R^n}{\norm{x}_1\le1} \ .
\end{eqnarray*}

The $1$-norm $\norm{.}_1$ is the sum of the absolute values of the entries of a vector\index{L1 norm@$\ell_1$-norm}. The standard Euclidean norm is written as $\norm{.}$.

For a given matrix $A\in\R^{d\times n}$, we write $A_i\in\R^d$, $i\in[1..n]$ for its columns.
We use the notation $AS:= \SetOf{Ax}{x\in S}$ for subsets $S\subseteq\R^d$ and matrices $A$. The convex hull of a set $S$ is written as $\conv(S)$\index{convex hull}.
By $\zero$ and $\one$ we denote the all-zero and all-ones vectors in $\R^n$, and~$\id_n$ is the $n\times n$ identity matrix. We write $(A|B)$ for the horizontal concatenation of two matrices $A,B$.

%
\section{Linear Classifiers and Support Vector Machines}\label{sec:SVMproperties}

Linear classifiers\index{linear classifier}\index{classification} have become the standard workhorse for many machine learning problems. Suppose we are given $n$ datapoints $X_i\in\R^d$, together with their binary labels $y_i\in \{\pm1\}$, for $i\in[1..n]$.

As we illustrate in Figure~\ref{fig:SVMsep}, a linear classifier is a hyperplane\index{hyperplane} that partitions the space $\R^d$ into two parts, such that hopefully each point with a positive label will be on one side of the plane, and the points of negative label will be on the other site.
Writing the classifier as the normal vector\index{normal vector} $w$ of that hyperplane, we can formally write this separation\index{separation} as $X_i^T w > 0$ for those points~$i$ with $y_i\ =+1$, and  $X_i^T w < 0$ if $y_i\ =-1$, assuming we consider hyperplanes that pass through the origin.

\begin{figure}[h]
\centerline{
\includegraphics[width=0.95\textwidth]{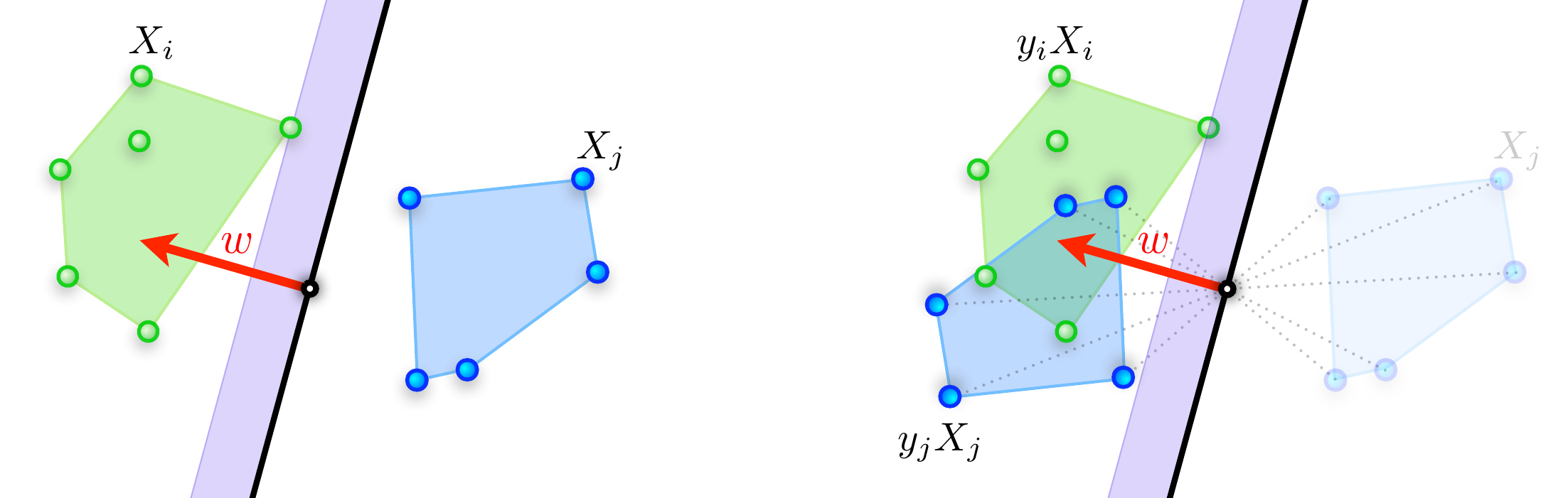}
}
\caption[Linear Classifier]{A linear classifier: illustration of the separation of two point classes, for a normal vector $w$. Here we highlight two points $i,j$ of different labels, with $y_i=+1$ and $y_j=-1$.}
\label{fig:SVMsep}
\end{figure}

\subsection{Support Vector Machines}
The popular concept of \emph{support vector machines} (SVM)\index{Support Vector Machine}\index{SVM} is precisely this linear classifier idea as we have introduced above, with one important addition: instead of being satisfied with any arbitrary one among all separating hyperplanes, we want to find the hyperplane which separates the two point classes by the best possible margin.

The \emph{margin}\index{margin} is defined as the smallest distance to the hyperplane among all the datapoints. Maximizing this margin over all possible linear classifiers~$w$ can be simply formalized as the following optimization problem:
\[
\max_{w\in\R^d} \, \min_i \, y_iX_i^T\frac{w}{\norm{w}} \ ,
\]
i.e. maximizing the smallest projection onto the direction $w$, among all datapoints.

The SVM optimization problem~(\ref{eq:classifier}) that we defined in the introduction is very closely related to this form of optimization problem, as follows. If we take the Lagrange dual of problem~(\ref{eq:classifier}), we exactly obtain a problem of this margin maximization\index{hard-margin SVM} type, namely
\vspace{-1mm}
\begin{equation}\label{eq:svmProj}
\max_{x\in\Simplex} \, \min_i \, A_i^T\frac{Ax}{\norm{Ax}} \ .
\end{equation}
Here we think of the columns of $A$ being the SVM datapoints with their signs $A_i = y_iX_i$.
For an overview of Lagrange duality\index{Lagrange duality}, we refer the reader to \cite[Section 5]{Boyd:2004uz}, or see for example \cite[Appendix A]{GartnerJaggi:2009} for the SVM case here.
An alternative and slightly easier way to obtain this dual problem is to compute the simple ``linearization'' dual function as in \cite{Hearn:1982ee,Jaggi:2013wg}, which avoids the notion of any dual variables. When starting from the non-squared version of the SVM problem~(\ref{eq:classifier}), this also gives the same formulation (\ref{eq:svmProj}).

No matter if we use problem formulation~(\ref{eq:classifier}) or (\ref{eq:svmProj}),
any feasible weight vector $x$ readily gives us a candidate classifier $w=Ax$, represented as a convex combination of the datapoints, because the weight vectors $x\in\Simplex\subset\R^n$ lie in the simplex. The datapoints corresponding to non-zero entries in $x$ are called the \emph{support vectors}\index{support vectors}.

Several other SVM variants of practical interest also have the property that their dual optimization problem is of the form~(\ref{eq:classifier}), as we will discuss in the next Subsection~\ref{sec:LassoLeSVM}.

\paragraph{Kernelization.}\index{kernelization}\index{kernel trick}
A crucial and widely used observation is that both optimization problems (\ref{eq:classifier}) and (\ref{eq:svmProj}) are formulated purely in terms of the inner products between pairs of datapoints $A_i := y_iX_i$, meaning that they can directly be optimized in the \emph{kernel} case \cite{cortes95softmargin}, provided that we only have access to the entries of the matrix $A^TA \in\R^{n\times n}$, but not the explicit features $A\in\R^{d\times n}$.
The matrix $K = A^TA$ is called the \emph{kernel matrix} in the literature. Using this notation, the SVM dual optimization problem (\ref{eq:classifier}) becomes
\[
  \min_{x\in\Simplex} \, x^TKx \ .
\]

\paragraph{Separation, and Approximate Solutions.}
It is natural to measure the quality of an approximate solution\index{approximate solution} $x$ to the SVM problem as the attained margin, which is precisely the attained value in the above problem~(\ref{eq:svmProj}).

\begin{definition}\label{def:weaklySep}
A vector $w\in\R^d$ is called \emph{$\sigma$-weakly-separating}\index{weakly-separating}\index{separation} for the SVM instance~(\ref{eq:classifier}) or~(\ref{eq:svmProj}) respectively, for a parameter $\sigma\ge0$, if it holds that
\[
A_i^T\frac{w}{\norm{w}} \ge \sigma ~~~\forall i\ ,
\]
meaning that $w$ attains a margin of separation of $\sigma$.
\end{definition}
The margin value $\sigma$ in this definition, or also the objective in~(\ref{eq:svmProj}), can be interpreted as a useful certificate for the attained optimization quality as follows.
If we take some $x$ (the separating vector now being interpreted as $w = Ax$), then the \emph{duality gap} is given by the difference of the margin value from the corresponding objective $\norm{Ax}$ in problem~(\ref{eq:classifier}). This gap function is a certificate for the approximation quality (since the unknown optimum must lie within the gap), which makes it a very useful stopping criterion for optimizers, see e.g. \cite{Hearn:1982ee,GartnerJaggi:2009,Clarkson:2010ue,Jaggi:2013wg}.

The simple perceptron\index{perceptron} algorithm~\cite{Rosenblatt:1958} is known to return a $\sigma$-weakly-separating solution to the SVM after $O(1/\varepsilon^2)$ iterations, for $\varepsilon := \sigma^*-\sigma$ being the additive error, if~$\sigma^*$ is the optimal solution to~(\ref{eq:classifier}) and~(\ref{eq:svmProj}).


\subsection{Soft-Margin SVMs}\label{sec:softSVM}\index{soft-margin SVM}
For the successful practical application of SVMs, the soft-margin concept of tolerating outliers is of central importance.
Here we recall that also the soft-margin SVM variants using $\ell_2$-loss, with regularized offset or no offset, both in the one-class and the two-class case, can be formulated in the form~(\ref{eq:classifier}). This fact is known in the SVM literature \cite{Scholkopf:2002ta,Keerthi:2000tj,Tsang:2005up}, and can be formalized as follows.

The two-class soft-margin SVM with squared loss ($\ell_2$-loss), without offset term, is given by the primal optimization problem
\begin{equation}\label{eq:softSVMprimal}
\begin{array}{rl}
   \displaystyle\min_{\substack{\bar w \in \R^d,\ \rho\in\R,\\ \xi\in\R^n}} & \frac12 \norm{\bar w}^2 - \rho + \frac{C}{2}\sum_i \xi_i^2\vspace{-2pt}  \\
    s.t. &  y_i \cdot \bar w^T X_i \ge \rho - \xi_i \ \ \forall i \in [1..n] \ .\!\!
\end{array}
\end{equation}
For each datapoint, we have introduced a slack-variable\index{slack-variable} $\xi_i$ which is penalized in the objective function in the case that the point does violate the margin.
Here $C>0$ is the regularization parameter\index{regularization parameter}, steering the tradeoff between large margin and punishing outliers. In the end, $\rho/\!\norm{\bar w}$ will be the attained margin of separation. 
Note that in the classical SVM formulation, the margin parameter~$\rho$ is usually fixed to one instead, while~$\rho$ is explicitly used in the equivalent $\nu$-SVM formulation known in the literature, see e.g. \cite{Scholkopf:2002ta}.
The equivalence of the soft-margin SVM dual problem\index{SVM dual optimization problem} to the optimization problem~(\ref{eq:classifier}) is stated in the following Lemma:
\begin{lemma}\label{lem:softSVM}
The dual of the soft-margin SVM (\ref{eq:softSVMprimal}) is an instance of the classifier formulation~(\ref{eq:classifier}), that is \ 
$
  \min_{x\in\Simplex} \, \norm{Ax}^2
$
, with\vspace{-1pt}
\[
A := {Z\choose\frac1{\sqrt{C}}\id_n} \in \R^{(d+n)\times n} \vspace{-1pt}
\]
where the data matrix $Z\in\R^{d\times n}$ consists of the $n$ columns $Z_i := y_iX_i$.
\end{lemma}
\begin{proof}
Given in Appendix~\ref{sec:SVMderivation} for completeness, using standard Lagrange duality.\index{Lagrange duality}
\end{proof}

Not all SVM variants have our desired structure of the dual problem. 
For example, the hinge-loss\index{hinge-loss} (or $\ell_1$-loss) SVM refers to the setting where the outliers are penalized according to their margin-violation $\xi_i$, not the squared values $\xi_i^2$ as we use here. Changing the loss function in the primal optimization problem also affects the dual problem, so that the dual of the $\ell_1$-loss SVM is \emph{not} of the form~(\ref{eq:classifier}).
However in practice, it is known that both SVM variants with $\ell_1$- or $\ell_2$-loss do perform similarly well for most applications~\cite{Lee:2001vf,Chang:2008up}.

\paragraph{Obtaining a Weakly-Separating Vector for the $\ell_2$-loss Soft-Margin SVM.}\index{weakly-separating}
By the above lemma, we observe that a weakly-separating vector is trivial to obtain for the $\ell_2$-loss SVM. This holds without any assumptions on the original input data $(X_i,y_i)$.
We set $w:= {\zero\choose\frac1{\sqrt{n}}\one} \in \R^{d+n}$ to the all-one vector only on the second block of coordinates, rescaled to unit length. Clearly, this direction~$w$ attains a separation margin of 
\[
A_i^T\frac{w}{\norm{w}} = {y_iX_i\choose\frac1{\sqrt{C}}\unit_i}^T {\zero\choose\frac1{\sqrt{n}}\one} = \frac1{\sqrt{nC}} > 0
\]
for all points~$i$ in Definition~\ref{def:weaklySep}.
%

\paragraph{Incorporating an Offset Term.}\index{offset (SVM)}\index{bias (SVM)}
Our above SVM formulation also allows the use of an \emph{offset} (or bias) variable $b\in\R$ to obtain a classifier that does not necessarily pass through the origin.
Formally, the separation constraints then become $y_i \cdot (w^T X_i + b) \ge \rho - \xi_i \ \ \forall i \in [1..n]$.
There is a well-known trick to efficiently emulate such an offset parameter while still using our formulation~(\ref{eq:softSVMprimal}), by simply increasing the dimensionality of~$X_i$ and~$w$ by one, and adding a fixed value of one as the last coordinate to each of the datapoints~$X_i$, see e.g. \cite{Keerthi:2000tj,Tsang:2005up}. As a side-effect, the offset $b^2$ is then also regularized in the new term~$\norm{w}^2$. Nevertheless, if desired, the effect of this additional regularization can  be made arbitrary weak by re-scaling the fixed additional feature value from one to a larger value. 

\paragraph{One-Class SVMs.}\index{one-class SVM}
All mentioned properties in this section also hold for the case of \emph{one-class SVMs}, by setting all labels $y_i$ to one, resulting in the same form of optimization problems (\ref{eq:classifier}) and (\ref{eq:svmProj}). One class SVMs are popular for example for anomaly or novelty detection applications.


\newpage
\section{The Equivalence}

\subsection{Warm-Up: Equivalence between SVM and Non-Negative Lasso}
Before we investigate the ``real'' Lasso problem~(\ref{eq:lasso}) in the next two subsections, we will warm-up by considering the non-negative variant~(\ref{eq:posLasso}).
It is a simple observation that the non-negative Lasso~(\ref{eq:posLasso})\index{non-negative Lasso} is directly equivalent to the dual SVM problem~(\ref{eq:classifier}) by a translation:

\paragraph{Equivalence by Translation.}\index{translation}
Given a non-negative Lasso instance~(\ref{eq:posLasso}), we can translate each column vector of the matrix $A$ by the vector $-b$. Doing so, we precisely obtain an SVM instance~(\ref{eq:classifier}), with the data matrix being 
\[
\tilde A := A - b\one^T \in\R^{d\times n} \ .
\]
Here we have crucially used the simplex domain, ensuring that $b\one^Tx = b$ for any $x\in\Simplex$.
To summarize, for those two optimization problems, the described translation precisely preserves all the objective values of all feasible points for both problems~(\ref{eq:posLasso}) and~(\ref{eq:classifier}), that is for all $x\in\Simplex$. 
This is why we say that the problems are equivalent.

The reduction in the other direction --- i.e. reducing an SVM instance~(\ref{eq:classifier}) to a non-negative Lasso instance~(\ref{eq:posLasso}) --- 
is trivial by choosing $b:=\zero$.
\\

Now to relate the SVM to the ``real'' Lasso, the same translation idea is of crucial importance. We explain the two reductions in the following subsections.

%
\subsection{\textrm{(Lasso $\preceq$ SVM)}: Given a Lasso Instance, Constructing an Equivalent SVM Instance}\label{sec:LassoLeSVM}\index{equivalent SVM instance}

(This reduction is significantly easier than the other direction.)

\paragraph{Parameterizing the $\ell_1$-Ball as a Convex Hull.}
One of the main properties of polytopes --- if not the main one --- is that every polytope can be represented as the convex hull of its vertices~\cite{Ziegler:1995td}. When expressing an arbitrary point in the polytope as a convex combination of some vertices, this leads to the standard concept of \emph{barycentric} coordinates.

In order to represent the $\ell_1$-ball $\LOneBall$ by a simplex $\Simplex$, this becomes particularly simple.
The $\ell_1$-ball $\LOneBall$ is the convex hull of its $2n$ vertices, which are $\SetOf{\pm\unit_i}{i\in[1..n]}$, illustrating why $\LOneBall$ is also called the cross-polytope.

The barycentric representation of the $\ell_1$-ball therefore amounts to the simple trick of using two non-negative variables to represent each real variable, which is standard for example when writing linear programs in standard form. For $\ell_1$ problems such as the Lasso, this representation was known very early~\cite{Bloomfield:1983tn,Chen:1998hm}.
Formally, any $n$-vector $x_\lOne\in\LOneBall$ can be written as 
\[
x_\lOne = (\id_n \,| {-\id_n}) x_\simpl ~\text{ for }~ x_\simpl\in\Simplex \subset \R^{2n} \ .
\]
Here $x_\simpl$ is a $2n$-vector, and we have used the notation $(A|B)$ for the horizontal concatenation of two matrices $A,B$.

Note that the barycentric representation is \emph{not} a bijection in general, as there can be several $x_\simpl\in\Simplex\subset\R^{2n}$ representing the same point $x_\lOne\in\R^n$.


\paragraph{The Equivalent SVM Instance.}
Given a Lasso instance of the form~(\ref{eq:lasso}), that is, $\min_{x\in\LOneBall}  \norm{Ax-b}^2$, we can directly parameterize the $\ell_1$-ball by the $2n$-dimensional simplex as described above.
By writing $(\id_n \,| {-\id_n}) x_\simpl$ for any $x\in\LOneBall$, the objective function becomes $\norm{(A \,| {-A}) x_\simpl-b}^2$. This means we have obtained the equivalent non-negative regression problem of the form~(\ref{eq:posLasso}) over the domain $x_\simpl\in\Simplex$ which, by our above remark on translations, is equivalent to the SVM formulation~(\ref{eq:classifier}), i.e.
\[
\min_{x_\simpl\in\Simplex} \norm{\tilde Ax_\simpl}^2 \ ,
\]
where the data matrix is given by
\[
\tilde A := (A \,| {-A}) - b\one^T \in\R^{d\times 2n} \ .
\]
The additive rank-one term $b\one^T$ for $\one\in\R^{2n}$ again just means that the vector $b$ is subtracted from each original column of $A$ and $-A$, resulting in a translation\index{translation} of the problem. So we have obtained an equivalent SVM instance consisting of $2n$ points in $\R^d$.

Note that this equivalence not only means that the optimal solutions of the Lasso and the SVM coincide, but indeed gives us the correspondence of all feasible points, preserving the objective values: for any points solution $x\in\R^{n}$ to the Lasso, we have a feasible SVM point $x_\simpl\in\Simplex\subset\R^{2n}$ of the \emph{same} objective value, and vice versa.

\subsection{\textrm{(SVM $\preceq$ Lasso)}: Given an SVM Instance, Constructing an Equivalent Lasso Instance}\label{sec:SVMLeLasso}\index{equivalent Lasso instance}

This reduction is harder to accomplish than the other direction we explained before. Given an instance of an SVM problem~(\ref{eq:classifier}), we suppose that we have a (possibly non-optimal) $\sigma$-weakly-separating\index{weakly-separating} vector $w\in\R^d$ available, for some (small) value $\sigma>0$.
Given $w$, we will demonstrate in the following how to construct an equivalent Lasso instance~(\ref{eq:lasso}).

Perhaps surprisingly, such a weakly-separating vector $w$ is trivial to obtain for the $\ell_2$-loss soft-margin SVM, as we have observed in Section~\ref{sec:softSVM} (even if the SVM input data is not separable).
Also for hard-margin SVM variants, finding such a weakly-separating vector for a small $\sigma$ is still significantly easier than the final goal of obtaining a near-perfect $(\sigma^*-\varepsilon)$-separation for a small precision $\varepsilon$.
It corresponds to running an SVM solver (such as the perceptron algorithm) for only a constant number of iterations. In contrast, obtaining a better $\varepsilon$-accurate solution by the same algorithm would require $O(1/\varepsilon^2)$ iterations, as mentioned in Section~\ref{sec:SVMproperties}.

\paragraph{The Equivalent Lasso Instance.}
Formally, we define the Lasso instance~$(\tilde A,\tilde b)$ as the translated SVM datapoints \[
\tilde A := \SetOf{A_i + \tilde b}{i\in[1..n]}
\]
together with the right hand side 
\[
\tilde b:= - \frac{w}{\norm{w}}\cdot\frac{D^2}{\sigma} \ .
\]

Here $D>0$ is a strict upper bound on the length of the original SVM datapoints, i.e. $\norm{A_i}<D\ \ \forall i$.

By definition of $\tilde A$, the resulting new Lasso objective function is 
\begin{equation}\label{eq:lassoObj}
\norm{\tilde Ax-\tilde b}
~=~\norm{(A+\tilde b\one^T)x-\tilde b}
~=~\norm{Ax+(\one^Tx-1)\tilde b} \ .\!\!
\end{equation}
Therefore, this objective coincides with the original SVM objective~(\ref{eq:classifier}), for any $x\in\Simplex$ (meaning that $\one^Tx=1$). However, this does not necessarily hold for the larger part of the Lasso domain when $x \in \LOneBall\setminus\Simplex$.
In the following discussion and the main Theorem~\ref{thm:LassoForSVM}, we will prove that all those candidates $x \in \LOneBall\setminus\Simplex$ can be discarded from the Lasso problem, as they do not contribute to any optimal solutions.

As a side-remark, we note that the quantity $\frac{D}{\sigma}$ that determines the magnitude of our translation is a known parameter in the SVM literature.
\cite{Burges:1998hg,Scholkopf:2002ta} have shown that the VC-dimension of an SVM, a measure of ``difficulty'' for the classifier, is always lower than $\frac{D^2}{\sigma^2}$.
Note that by the definition of separation, $\sigma\leq D$ always holds.


\paragraph{Geometric Intuition.}\index{geometric interpretation}
Geometrically, the Lasso problem~(\ref{eq:lasso}) is to compute the smallest Euclidean distance of the set $A\LOneBall$ to the point $b\in\R^d$. 
On the other hand the SVM problem --- after translating by $b$ --- is to minimize the distance of the smaller set $A\Simplex \subset A\LOneBall$ to the point $b$.
Here we have used the notation $AS:= \SetOf{Ax}{x\in S}$ for subsets $S\subseteq\R^d$ and linear maps $A$ (it is easy to check that linear maps do preserve convexity of sets, so that $\conv(AS) = A\conv(S)$).

Intuitively, the main idea of our reduction is to mirror our SVM points~$A_i$ at the origin, so that both the points and their mirrored copies --- and therefore the entire larger polytope~$A\LOneBall$ --- do end up lying ``behind'' the separating SVM margin. 
The hope is that the resulting Lasso instance will have all its optimal solutions being non-negative, and lying in the simplex. 
Surprisingly, this can be done, and we will show that all SVM solutions are preserved (and no new solutions are introduced) when the feasible set $\Simplex$ is extended to~$\LOneBall$.
In the following we will formalize this precisely, and demonstrate how to translate along our known weakly-separating vector $w$ so that the resulting Lasso problem 
will have the same solution as the original SVM.

\paragraph{Properties of the Constructed Lasso Instance.}
The following theorem shows that for our constructed Lasso instance, all interesting feasible points are contained in the simplex $\Simplex$. By our previous observation~(\ref{eq:lassoObj}), we already know that all those candidates are feasible for both the Lasso~(\ref{eq:lasso}) and the SVM~(\ref{eq:classifier}), and obtain the same objective values in both problems.

In other words, we have a one-to-one correspondence between all feasible points for the SVM~(\ref{eq:classifier}) on one hand, and the subset $\Simplex \subset \LOneBall$ of feasible points of our constructed Lasso instance~(\ref{eq:lasso}), preserving all objective values. 
Furthermore, we have that in this Lasso instance, all points in $\LOneBall\setminus\Simplex$ are strictly worse than the ones in $\Simplex$.
Therefore, we have also shown that all optimal solutions must coincide.

\begin{theorem}\label{thm:LassoForSVM}
For any candidate solution $x_\lOne\in\LOneBall$ to the Lasso problem~(\ref{eq:lasso}) defined by $(\tilde A,\tilde b)$, 
there is a feasible vector $x_\simpl\in\Simplex$ in the simplex, of the same or better Lasso objective value $\gamma$.\\
Furthermore, this $x_\simpl\in\Simplex$ attains the same objective value $\gamma$ in the original SVM problem~(\ref{eq:classifier}).
\end{theorem}

On the other hand, every $x_\simpl\in\Simplex$ 
is of course also feasible for the Lasso, and attains the same objective value there, again by~(\ref{eq:lassoObj}).

\begin{proof}
The proof follows directly from the two main facts given in Propositions~\ref{prop:Flipping} and~\ref{prop:Scaling} below, which state that ``flipping negative signs improves the objective'', and that ``scaling up improves for non-negative vectors'', respectively.
We will see below as why these two facts hold, which is precisely by the choice of the translation $\tilde b$ along a weakly separating vector $w$, in order to define our Lasso instance.

We assume that the given $x_\lOne$ does not already lie in the simplex. Now by applying 
Propositions~\ref{prop:Flipping} and \ref{prop:Scaling}, we obtain $x_\simpl\in\Simplex$, of a strictly better objective value $\gamma$ for problem~(\ref{eq:posLasso}). 
By the observation~(\ref{eq:lassoObj}) about the Lasso objective, we know that the original SVM objective value attained by this $x_\simpl$ is equal to~$\gamma$.
\end{proof}

\begin{proposition}[Flipping negative signs improves the objective]\label{prop:Flipping}
Consider the Lasso problem~(\ref{eq:lasso}) defined by $(\tilde A,\tilde b)$, and assume that $x_\lOne\in\LOneBall$ has some negative entries. 

Then there is a strictly better solution $x_{\fSimpl}\in\FilledSimplex$ having only non-negative entries.
\end{proposition}

\begin{proof}
We are given $x_\lOne\ne0$, having at least one negative coordinate. Define $x_{\fSimpl}\ne\zero$ as the vector you get by flipping all the negative coordinates in $x_\lOne$. We define $\delta \in \FilledSimplex$ to be the difference vector corresponding to this flipping, i.e. $\delta_i := -(x_\lOne)_i$ if $(x_\lOne)_i<0$, and $\delta_i :=0$ otherwise, so that  $x_{\fSimpl} := x_\lOne+2\delta$ gives $x_{\fSimpl}\in\FilledSimplex$.
We want to show that with respect to the quadratic objective function, $x_{\fSimpl}$ is strictly better than $x_\lOne$. We do this by showing that the following difference in the objective values is strictly negative:
\begin{eqnarray*}
\norm{\tilde Ax_{\fSimpl} - \tilde b}^2 - \norm{\tilde Ax_\lOne - \tilde b}^2
&=& \norm{c+d}^2 - \norm{c}^2 \\
&=&  c^Tc + 2c^Td + d^Td - c^Tc
\ =\   (2c+d)^Td\\
&=&  4 (\tilde Ax_\lOne - \tilde b +  \tilde A\delta)^T \tilde A\delta\\
&=&  4 (\tilde A(x_\lOne+\delta) - \tilde b)^T \tilde A\delta
\end{eqnarray*}
where in the above calculations we have used that $\tilde Ax_{\fSimpl} = \tilde Ax_\lOne +2\tilde A\delta$, and we substituted $c:= \tilde Ax_\lOne - \tilde b$ and $d:= 2\tilde A\delta$.
Interestingly, $x_\lOne+\delta \in \FilledSimplex$, since this addition just sets all previously negative coordinates to zero.

The proof then follows from Lemma~\ref{lem:tecInnerPositive} below.
\end{proof}

\begin{proposition}[Scaling up improves for non-negative vectors]\label{prop:Scaling}
Consider the Lasso problem~(\ref{eq:lasso}) defined by~$(\tilde A,\tilde b)$, and assume that $x_{\fSimpl}\in\FilledSimplex$ has $\norm{x_{\fSimpl}}_1<1$. 

Then we obtain a strictly better solution $x_\simpl\in\Simplex$ by linearly scaling $x_{\fSimpl}$.
\end{proposition}
\begin{proof}
The proof follows along similar lines as the above proposition. We are given $x_{\fSimpl}\ne0$ with $\norm{x_{\fSimpl}}_1<1$. 
Define $x_\simpl$ as the vector we get by scaling up $x_\simpl := \lambda x_{\fSimpl}$ by $\lambda>1$ such that $\norm{x_\simpl}_1=1$. 
We want to show that with respect to the quadratic objective function, $x_\simpl$ is strictly better than~$x_{\fSimpl}$. As in the previous proof, we again do this by showing that the following difference in the objective values is strictly negative:
\begin{eqnarray*}
\norm{\tilde Ax_\simpl - \tilde b}^2 - \norm{\tilde Ax_{\fSimpl} - \tilde b}^2
&=& \norm{c+d}^2 - \norm{c}^2 \\
&=&  c^Tc + 2c^Td + d^Td - c^Tc
\ =\   (2c+d)^Td\\
&=&  \lambda' (2\tilde Ax_{\fSimpl} - 2b +  \lambda'\tilde Ax_{\fSimpl})^T \tilde Ax_{\fSimpl}\\
&=&  \textstyle2 \lambda' (\tilde A\big(1+\frac{\lambda'}{2}\big)x_{\fSimpl} - \tilde b)^T \tilde Ax_{\fSimpl}
\end{eqnarray*}
where in the above calculations we have used that $\tilde Ax_\simpl = \lambda\tilde Ax_{\fSimpl}$ for $\lambda>1$, and we substituted $c:= \tilde Ax_{\fSimpl} - \tilde b$ and $d:= \tilde Ax_\simpl - \tilde Ax_{\fSimpl} = (\lambda-1)\tilde Ax_{\fSimpl} =: \lambda'\tilde Ax_{\fSimpl}$ for $\lambda' := \lambda-1 >0$.
Note that $x_\simpl := (1+\lambda') x_{\fSimpl} \in\Simplex$ so $\big(1+\frac{\lambda'}{2}\big)x_{\fSimpl} \in\FilledSimplex$.

The proof then follows from Lemma~\ref{lem:tecInnerPositive} below.
\end{proof}

\begin{definition}
For a given axis vector $w\in\R^d$, the \emph{cone}\index{cone} with axis $w$, angle $\alpha\in(0,\frac\pi2)$ with tip at the origin is defined as $\cone(w,\alpha) := \SetOf{x\in\R^d}{\measuredangle(x,w)\le\alpha}$, or equivalently $\frac{x^Tw}{\norm{x}\norm{w}}\ge \cos\alpha $.
By $\intcone(w,\alpha)$ we denote the interior of the convex set $\cone(w,\alpha)$, including the tip $\zero$.
\end{definition}

\begin{figure}[htb]
\centerline{
\includegraphics[width=0.7\textwidth]{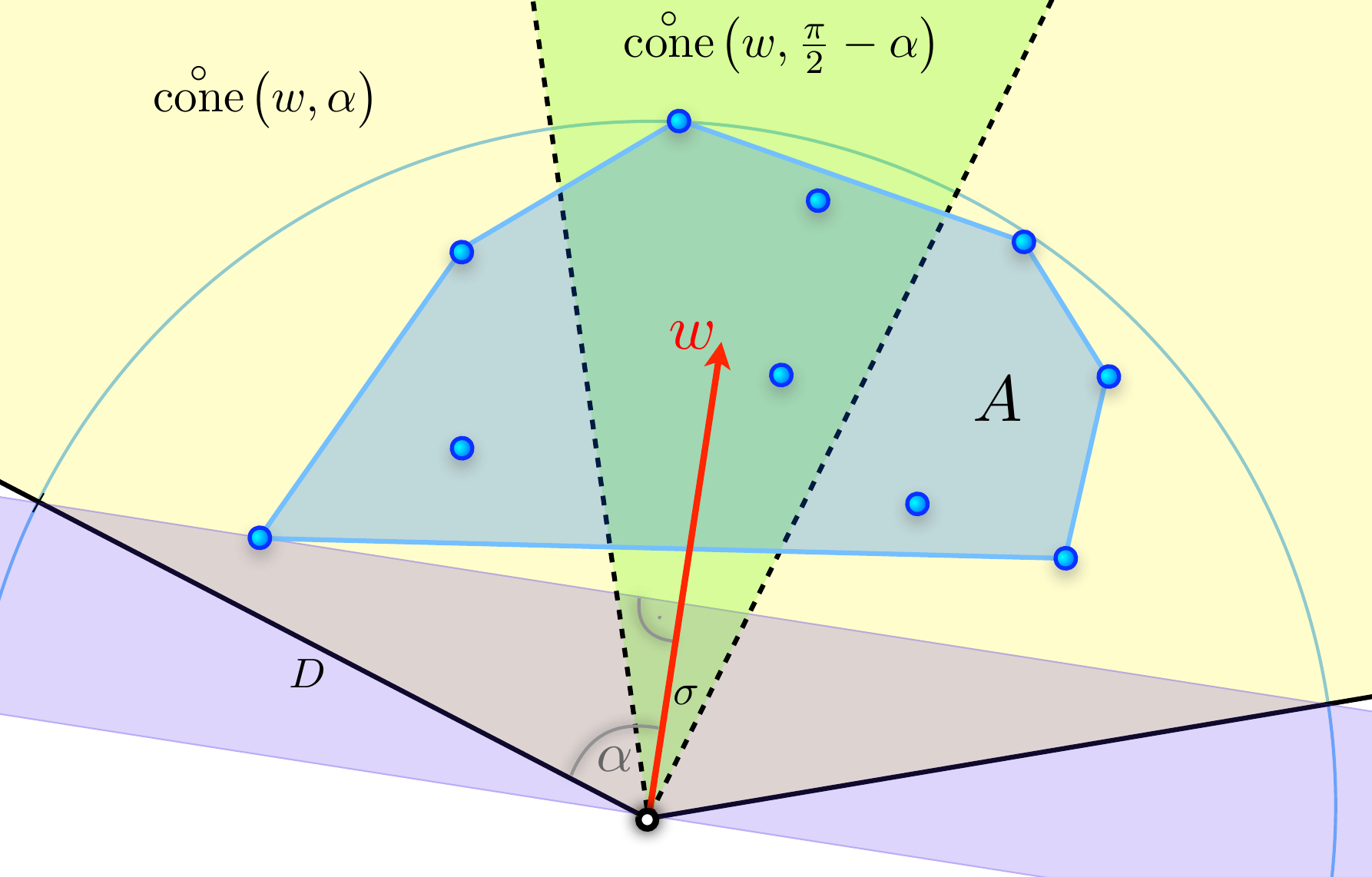}
}
\caption[Cone of Separating Vectors]{Illustration of the separation idea from Lemma \ref{lem:sepCone}, showing the cone of vectors that are still weakly-separating for the set of points $A$. Here we used the angle $\alpha := \arccos(\frac{\sigma}{D})$.}
\label{fig:sepCone}
\end{figure}

\begin{lemma}[Separation]\label{lem:sepCone}
Let $w$ be some $\sigma$-weakly-separating vector for the SVM~(\ref{eq:classifier}) for $\sigma>0$.
Then 
\begin{enumerate}
  \item[i)] $A\FilledSimplex \subseteq \intcone(w,\arccos(\frac{\sigma}{D}))$
  \item[ii)] Any vector in $\cone(w,\arcsin(\frac{\sigma}{D}))$ is still $\sigma'$-weakly-separating for $A$ for some $\sigma'>0$.
\end{enumerate}
\end{lemma}
\begin{proof}
i) Definition \ref{def:weaklySep} of weakly separating, and using that $\norm{A_i} < D$.

ii) For any unit length vector $v\in\cone(w,\arcsin(\frac{\sigma}{D}))$, every other vector having a zero or negative inner product with this $v$ must have angle at least $\frac\pi2-\arcsin(\frac{\sigma}{D}) = \arccos(\frac{\sigma}{D})$ with the cone axis $w$. However, by using i), we have $A\Simplex \subseteq \intcone(w,\arccos(\frac{\sigma}{D}))$, so every column vector of~$A$ must have strictly positive inner product with $v$, or in other words $v$ is $\sigma'$-weakly-separating for $A$ (for some $\sigma'>0$).
\end{proof}

\begin{figure}[htb]
\centerline{
\includegraphics[width=0.7\textwidth]{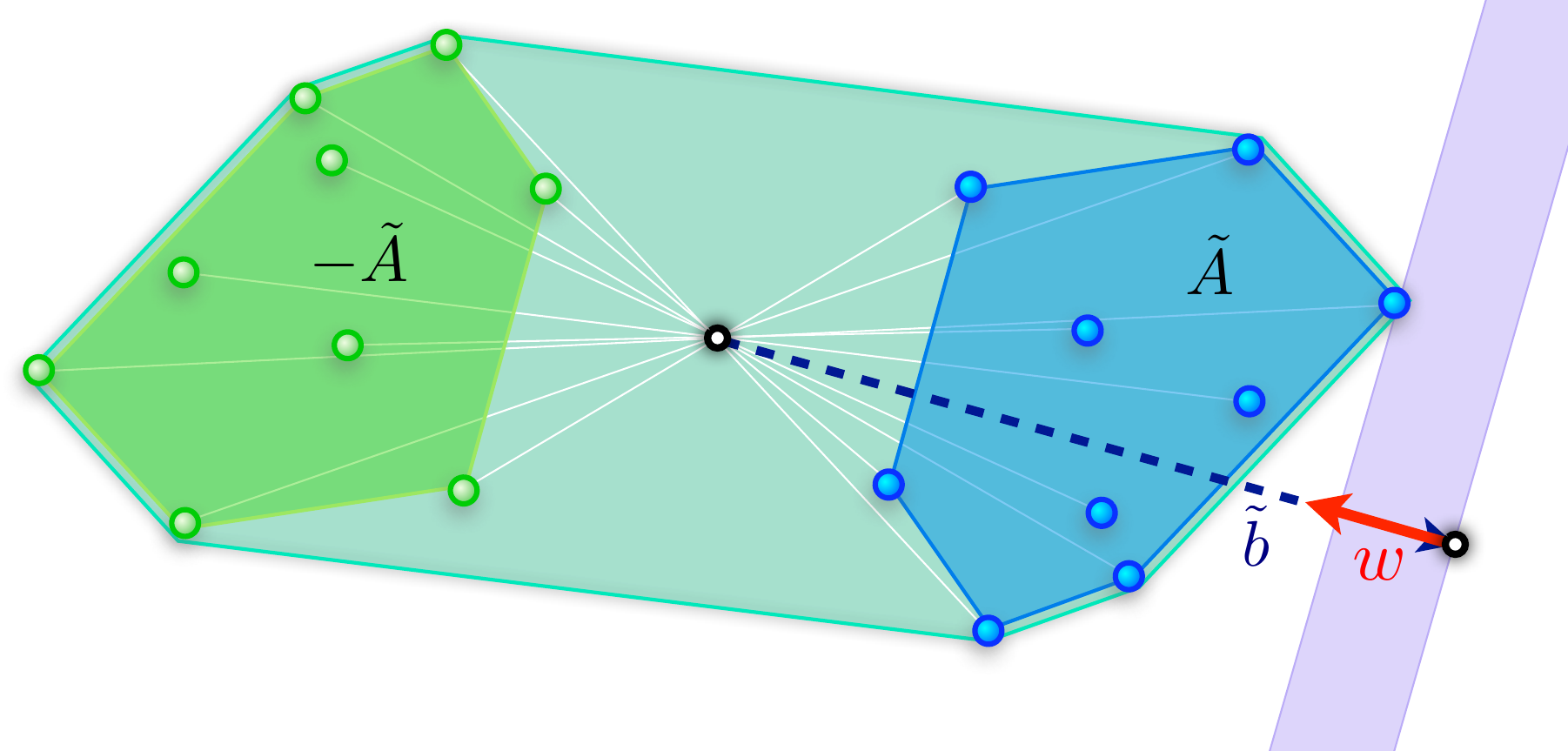}
}
\caption[Separation for the Lasso Instance]{Illustration of Lemma \ref{lem:tecInnerPositive}.
Recall that the translated points are defined by $\tilde A := \SetOf{A_i + \tilde b}{i\in[1..n]}$, where the translation is $\tilde b:= - \frac{w}{\norm{w}}\cdot\frac{D^2}{\sigma}$.}
\label{fig:lassoSep}
\end{figure}

\begin{lemma}\label{lem:tecInnerPositive}
Let $w$ be some $\sigma$-weakly-separating\index{weakly-separating}\index{separation} vector for the SVM for $\sigma>0$.
Then we claim that the translation by the vector $\tilde b:= - \frac{w}{\norm{w}}\cdot\frac{D^2}{\sigma}$ has the following properties.
For any pair of vectors $x,\delta\in\FilledSimplex,\delta\ne\zero$, we have that $(\tilde Ax - \tilde b)^T (-\tilde A\delta) > 0$.
\end{lemma}
\begin{proof}
(See also Figure \ref{fig:lassoSep}).
By definition of the translation $\tilde b$, we have that the entire Euclidean ball of radius $D$ around the point~$-\tilde b$ --- and therefore also the point set $-\tilde A\FilledSimplex$ and in particular $v:=-\tilde A\delta$ --- is contained in $\cone(w,\arcsin(\frac{\sigma}{D}))$. 
Therefore by Lemma~\ref{lem:sepCone} ii), $v$ is separating for~$A$, and by translation $v$ also separates $\tilde A$ from~$\tilde b$. This establishes the result $(\tilde Ax - \tilde b)^T v > 0$ for any $x\in \Simplex$.

To extend this to the case $x\in\FilledSimplex$, we observe that by definition of $\tilde b$, also the point $\zero-\tilde b$ has strictly positive inner product with $v$. Therefore the entire convex hull of $\tilde A\Simplex \,\cup\, \zero$ and thus the set $\tilde A\FilledSimplex$ has the desired property.
\end{proof}

\newpage
\section{Implications \& Remarks}

In the following, we will explain a handful of implications of the shown equivalence, by relating both algorithms as well as known theoretical results for the Lasso or the SVM also to the respective other method.


%

%
\subsection{Sublinear Time Algorithms for Lasso and SVMs}\index{sublinear time algorithms}
The recent breakthrough SVM algorithm of \cite{Clarkson:2010ue,Hazan:2011wq} in time $O(\varepsilon^{-2}(n+d)\log n)$ returns an $\varepsilon$-accurate solution to problem~(\ref{eq:classifier}). Here $\varepsilon$-accurate means $(\sigma^*-\varepsilon)$-weakly-separating. The running time of the algorithm is remarkable since for large data, it is significantly smaller than even the size of the input matrix, being $d\cdot n$. Therefore, the algorithm does not read the full input matrix~$\tilde A$.
More precisely, \cite[Corollary III.2]{Clarkson:2010ue} 
proves that the algorithm provides (with high probability) a solution $p^*\in\Simplex$ of additive error at most $\varepsilon$ to 
\[
\min_{p\in\Simplex} \ \max_{\substack{w\in\R^d,\\ \norm{w}\le1}} \  w^T\tilde Ap \ .
\]
This is a reformulation of $\min_{p\in\Simplex} p^T\tilde A^T\tilde Ap\,$, which is exactly our SVM problem~(\ref{eq:classifier}), since for given $p$, the inner maximum is attained when $w=\tilde Ap$. 
Therefore, using our simple trick from Section~\ref{sec:LassoLeSVM} of reducing any Lasso instance~(\ref{eq:lasso}) to an SVM~(\ref{eq:classifier}) (with its matrix $\tilde A$ having twice the number of columns as $A$), we directly obtain a sublinear time algorithm for the Lasso. Note that since the algorithm of \cite{Clarkson:2010ue,Hazan:2011wq} only accesses the matrix $\tilde A$ by simple entry-wise queries, it is not necessary to explicitly compute and store $\tilde A$ (which is a preprocessing that would need linear time and storage). Instead, every entry $\tilde A_{ij}$ that is queried by the algorithm can be provided on the fly, by returning the corresponding (signed) entry of the Lasso matrix~$A$, minus~$b_i$.

It will be interesting to compare this alternative algorithm to the recent more specialized sublinear time Lasso solvers in the line of work of \cite{CesaBianchi:2011wb,Hazan:2012wj}, which are only allowed to access a constant fraction of the entries (or features) of each row of $A$. If we use our proposed reduction here instead, the resulting algorithm from \cite{Clarkson:2010ue} has more freedom: it can (randomly) pick arbitrary entries of $A$, without necessarily accessing an equal number of entries from each row.

On the other hand, it is an open research question if a sublinear SVM algorithm exists which only accesses a constant small fraction of each datapoint, or of each feature of the input data.

\subsection{A Kernelized Lasso}\index{kernel trick}\index{kernelization}\index{kernelized Lasso}
Traditional kernel regression\index{kernel regression} techniques \cite{Smola:2004ba,Saunders:1998ws,Roth:2004gw} try to learn a real-valued function~$f$ from the space $\R^d$ of the datapoints, such that the resulting real value for each datapoint approximates some observed value. The regression model is chosen as a linear combination of the (kernel) inner products with few existing landmark datapoints (the support vectors).

Here, as we discuss a kernelization of the Lasso that is in complete analogy to the classical kernel trick for SVMs, our goal is different. We are not trying to approximate $n$ many individual real values (one for each datapoint, or row of $A$), but instead we are searching for a linear combination of our points in the kernel space, such that the resulting combination is \emph{close} to the lifted point~$b$, measured in the kernel space norm. Formally, suppose our kernel space~$\mathcal H$ is given by an inner product $\kappa(y,z) = \langle \Phi(y), \Phi(z) \rangle$ for some implicit mapping $\Phi: \R^d \rightarrow {\mathcal H}$. Then we define our kernelized variant of the Lasso as
\begin{equation}\label{eq:kernelLasso}
  \min_{x\in\LOneBall} \, \norm{\sum_{i} \Phi(A_i) x_i-\Phi(b)}_{\mathcal H}^2 \ .
\end{equation}
Nicely, analogous to the SVM case, also this objective function here is determined purely in terms of the pairwise (kernel) inner products $\kappa(\cdot,\cdot)$.

An alternative way to see this is to observe that our simple ``mirror-and-translate'' trick from Section~\ref{sec:LassoLeSVM} also works the very same way in any kernel space $\mathcal H$. Here, the equivalent SVM instance is given by the $2n$ new points $\SetOf{\pm\Phi(A_i)-\Phi(b)}{i\in[1..n]}\subset{\mathcal H}$.
The crucial observation is that the (kernel) inner product of any two such points is
\[
\begin{array}{rl}
&\left\langle s_i\Phi(A_i)-\Phi(b), s_j\Phi(A_j)-\Phi(b) \right\rangle \\
=& s_is_j \kappa(A_i,A_j) - s_i\kappa(A_i,b) - s_j\kappa(A_j,b) + \kappa(b,b) \ .
\end{array}
\]
Here $s_i,s_j \in\pm1$ are the signs corresponding to each point. Therefore we have completely determined the resulting $2n \times 2n$ kernel matrix $K$ that defines the kernelized SVM~(\ref{eq:classifier}), namely $\min_{x\in\Simplex}x^TKx$, which solves our equivalent Lasso problem~(\ref{eq:kernelLasso}) in the kernel space~$\mathcal H$. 

\textit{Discussion.}
While traditional kernel regression corresponds to a lifting\index{lifting} of the \emph{rows} of the Lasso matrix $A$ into the kernel space, our approach~(\ref{eq:kernelLasso}) by contrast is lifting the \emph{columns} of $A$ (and the r.h.s. $b$). We note that it seems indeed counter-intuitive to make the regression ``more difficult'' by artificially increasing the dimension of $b$. Using e.g. a polynomial kernel, this means that we also want the higher moments of $b$ to be well approximated by our estimated~$x$. On the other hand, increasing the dimension of $b$ naturally corresponds to adding more data rows (or measurements) to a classical Lasso instance~(\ref{eq:lasso}). 

In the light of the success of the kernel idea for the classification case with its existing well-developed theory, we think it will be interesting to relate these results to the above proposed kernelized version of the Lasso, and to study how different kernels will affect the solution $x$ for applications of the Lasso.
Using a different connection to SVMs, the early work of \cite{Girosi:1998dd} has studied a similar kernelization of the penalized version of the Lasso, see also \cite{Evgeniou:2000fx}. For applications in image retrieval, \cite{Thiagarajan:2012tb} has recently applied a similar Lasso kernelization idea.

%

%
\subsection{The Pattern of Support Vectors, in the View of Lasso Sparsity}\index{support vectors}\index{sparsity}
Using our construction of the equivalent Lasso instance for a given SVM, we can translate  sparsity results for the Lasso to understand the pattern of support vectors\index{pattern of support vectors} of SVMs.

The motivation here is that a small number of support vectors is crucial for the efficient application of SVMs in practice, in particular in the kernelized case, because the cost to evaluate the resulting classifier is directly proportional to the number of support vectors. Furthermore, the support vectors are the most informative points for the classification task, while the non-support vectors could be safely discarded from the problem.
%

\paragraph{Using Sparse Recovery Results.}
There has been a vast amount of literature studying the sparsity of solutions to the Lasso and related $\ell_1$-regularized methods, in particular the study of the sparsity of $x$ when $A$ and $b$ are from distributions with certain properties. For example, in the setting known as \emph{sparse recovery}\index{sparse recovery}, the goal is to approximately recover a sparse solution $x$ using a Lasso instance $A,b$ (consisting of only a small number of rows). Here $b$ is interpreted as a noisy or corrupted linear measurement $A\hat x$, and the unknown original $\hat x$ is sparse. Classical recovery results then show that under weak assumptions on $A$, $b$ and the sparsity of $\hat x$, the optimal Lasso solution $x$ must be identical to the (unknown) sparse~$\hat x$, see e.g.~\cite{Chen:1998hm,Porat:2012ud}. 

Now our construction of the equivalent Lasso instance for a given SVM allows us to translate such sparsity results to the pattern of SVM support vectors.
More precisely, any result that characterizes the Lasso sparsity\index{Lasso sparsity} for some distribution of matrices $A$ and suitable $b$, will also characterize the patterns of support vectors for the equivalent SVM instance (and in particular the number of support vectors).
This assumes that a Lasso sparsity result is applicable for the type of translation~$b$ that we have used to construct the equivalent Lasso instance. However, this is not hopeless. For example, existence of a weakly separating vector that is a sparse convex combination of the SVM datapoints is sufficient, since this results in a translation~$b$ that satisfies $b \propto A\hat x$ for a sparse weight vector $\hat x$.
It remains to investigate which distributions and corresponding sparsity results are of most practical interest from the SVM perspective, in order to guarantee a small number of support vectors.
%

\paragraph{Lasso Sparsity in the View of SVMs.}
In the very other direction, sparsity has also been studied for SVMs in the literature, for example in the work of \cite{Steinwart:2003wk}, which analyzes the asymptotic regime $n \rightarrow \infty$. Using the simpler one of our reductions, the same results also hold for the Lasso, when the number of variables $n$ grows.

\subsection{Screening Rules for Support Vector Machines}\index{screening}\index{screening rules}
For the Lasso, \emph{screening rules} have been developed recently. This approach consists of a pre-processing of the data $A$, in order to immediately discard those predictors\index{discard predictors} $A_i$ that can be guaranteed to be inactive for the optimal solution.
Provable guarantees for such rules were first obtained by~\cite{Ghaoui:2010tu}, and studied also in the later work \cite{Wang:2013uq}, or the heuristic paper~\cite{Tibshirani:2011iz}.

Translated to the SVM setting by our reduction, any such existing Lasso screening rule can be used to permanently discard input points before the SVM optimization is started. The screening rule then guarantees that any discarded point will not be a support vector, so the resulting optimal classifier remains unchanged. 
We are not aware of screening rules in the SVM literature so far, with the exception of the more recent paper of \cite{Ogawa:2013ul}.
\\

While the two previous subsections have mainly made use of the more complicated direction of our reduction (SVM $\preceq$ Lasso) from Section~\ref{sec:SVMLeLasso}, we can also gain some insights into the pattern of support vectors of SVMs by using the other (simpler) direction of reduction, as we will do next.

\subsection{Regularization Paths and Homotopy Methods}\label{sec:homotopy}\index{regularization path}\index{solution path methods}\index{homotopy methods}
For most machine learning methods --- including SVMs and Lasso --- one of the main hurdles in practice is the selection of the right free parameters. For SVMs and Lasso, the main question boils down on to how to select the best value for the regularization parameter\index{regularization parameter}, which determines the trade-off between the best fit of the model, and the model complexity.
For the SVM, this is the soft-margin parameter $C$, while in the Lasso, the regularization parameter is the value $r$ for the required bound $\norm{x}_1 \leq r$.

Since naive grid-search\index{grid-search} for the best parameter is error-prone and comes without guarantees between the grid-values, algorithms that follow the solution path --- as the regularization parameter changes --- have been developed for both SVMs \cite{Hastie:2004uj} and Lasso \cite{Osborne:2000kg,Efron:2004tz}, and have become popular in particular on the Lasso side.

In the light of our joint investigation of Lasso and SVMs here, we can gain the following insights on path methods for the two problems:

\paragraph{General Solution Path Algorithms.}
We have observed that both the primal Lasso (\ref{eq:lasso}) and the dual $\ell_2$-SVM (\ref{eq:classifier}) are in fact convex optimization problems over the simplex.
This enables us to apply the same solution path methods to both problems.
More precisely, for problems of the form $\min_{x\in\Simplex} f(x,t)$, general path following methods are available, that can maintain an approximation guarantee\index{approximate solution} along the entire path in the parameter $t$, as shown in \cite{Giesen:2012ik} and more recently strengthened by \cite{Giesen:2012uk}. These methods do apply for objective functions $f$ that are convex in $x$ and continuous in the parameter~$t$ (which in our case is $r$ for the Lasso and $C$ for the SVM).

\paragraph{Path Complexity.}\index{path complexity}
The exact solution path for the Lasso is known to be piecewise linear. 
However, the number of pieces, i.e. the complexity of the path, is not easy to determine.
The recent work of \cite{Mairal:2012wu} has constructed a Lasso instance $A\in\R^{d\times n}$, $b\in\R^n$, such that the complexity of the solution path (as the parameter~$r$ increases) is exponential in $n$. This is inspired by a similar result by \cite{Gartner:2012wi} which holds for the $\ell_1$-SVM.

Making use of the easier one of the reductions we have shown above, we ask if similar complexity worst-case results could also be obtained for the $\ell_2$-SVM.
For every constraint value $r>0$ for the Lasso problem, we have seen that the corresponding equivalent SVM instance (\ref{eq:classifier}) as constructed in Section~\ref{sec:LassoLeSVM} is $\min_{x\in\Simplex} \, \big\|\tilde A_{(r)}x\big\|^2$, with\vspace{-1mm}
\begin{equation}\label{eq:SVMmovingInstance}
\tilde A_{(r)} := r(A \,| {-A}) - b\one^T \in\R^{d\times 2n} \ .
\end{equation}

Therefore, we have obtained a hard-margin SVM, with the datapoints moving in the space~$\R^d$ as the Lasso regularization parameter $r$ changes. The movement is a simple linear rescaling by~$r$, relative to the reference point $b\in\R^d$.
The result of \cite{Mairal:2012wu} shows that essentially all sparsity patterns do occur in the Lasso solution as $r$ changes%
\footnote{
The result of \cite{Mairal:2012wu} applies to the penalized formulation of the Lasso, that is the solution path of $\min_{x\in\R^n} \norm{Ax-b}^2 + \lambda \norm{x}_1$ as the parameter $\lambda \in\R_+$ varies.
However, here we are interested in the constrained Lasso formulation, that is $\min_{x,\, \norm{x}_1\le r} \norm{Ax-b}^2$ where the regularization parameter is $r \in\R_+$.

Luckily, the penalized\index{penalized optimization} and the constrained\index{constrained optimization} problem formulations are known to be equivalent by standard convex optimization theory, and the have the same regularization paths, in the following sense.
For every choice of $\lambda \in\R_+$, there is a value for the constraint parameter $r$ such that the solutions to the two problems are identical (choose for example $r_{(\lambda)} := \big\|x^*_{(\lambda)} \big\|_1$ for some optimal $x^*_{(\lambda)}$, then the same vector is obviously also optimal for the constrained problem). This mapping from $\lambda$ to $r_{(\lambda)}$ is monotone. As the regularization parameter $\lambda$ is weakened (i.e. decreased), the corresponding constraint value $r_{(\lambda)}$ only grows larger.

The mapping in the other direction is similar: for every $r$, we know there is a value $\lambda_{(r)}$ (in fact this is the Lagrange multiplier of the constraint $\norm{x}_1\le r$), such that the penalized formulation has the same solution, and the mapping from $r$ to $\lambda_{(r)}$ is monotone as well.

Having both connections, it is clear that the two kinds of regularization paths must be the same, and that the support patterns that occur in the solutions --- as we go along the path --- are also appearing in the very same order as we track the respective other path.
}%
, i.e. that the number of patterns is exponential in $n$. 
For each pattern, we know that the SVM solution $x_\lOne = (\id_n \,| {-\id_n}) x_\simpl$ is identical to the Lasso solution, and in particular also has the same sparsity pattern. Therefore, we also have the same (exponential) number of different sparsity patterns in the simplex parameterization~$x_\simpl$ for the SVM (one could even choose more in those cases where the mapping is not unique).

To summarize, we have shown that a simple rescaling of the SVM data can have a very drastic effect, in that \emph{every} pattern of support vectors can potentially occur as this scaling changes. While our construction is still a worst-case result, note that the operation of rescaling is not unrealistic in practice, as it is similar to popular data preprocessing by re-normalizing the data, e.g. for zero mean and variance one. 
\\

However, note that the constructed instance (\ref{eq:SVMmovingInstance}) is a hard-margin SVM, with the datapoints moving as the parameter $r$ changes. It does not directly correspond to a soft-margin SVM, because the movement of points is different from changing the regularization parameter $C$ in an $\ell_2$-SVM.
As we have seen in Lemma~\ref{lem:softSVM}, changing $C$ in the SVM formulation~(\ref{eq:softSVMprimal}) has the effect that the datapoints~$\tilde A_{(C)}$ in the dual problem \ 
$
  \min_{x\in\Simplex} \, \big\| \tilde A_{(C)} x \big\|^2
$
move as follows (after re-scaling the entire problem by the constant factor $C$):
\[
\tilde A_{(C)} := {\scalebox{0.9}{$\sqrt{C}$}X\choose\id_n} \in \R^{(d+n)\times n} \vspace{-1pt}
\]

In conclusion, the reduction technique here does unfortunately not yet directly translate the regularization path of the Lasso to a regularization path for an $\ell_2$-SVM.
Still, we have gained some more insight as to how ``badly'' the set of SVM support vectors can change when the SVM data is simply re-scaled.
We hope that the correspondence will be a first step to better relate the two kinds of regularization paths in the future. Similar methods could also extend to the case when other types of parameters are varied, such as e.g. a kernel parameter.

\clearpage
\section{Conclusions}
We have investigated the relation between the Lasso and SVMs, and constructed equivalent instances of the respective other problem. While obtaining an equivalent SVM instance for a given Lasso is straightforward, the other direction is slightly more involved in terms of proof, but still simple to implement, in particular e.g. for $\ell_2$-loss SVMs. 

The two reductions allow us to better relate and compare many existing algorithms for both problems. Also, it can be used to translate a lot of the known theory for each method also to the respective other method.
In the future, we hope that the understanding of both types of methods can be further deepened by using this correspondence.



\appendix
\section{Some Soft-Margin SVM Variants that are Equivalent to~(\ref{eq:classifier})}
\label{sec:SVMderivation}

We include the derivation of the dual formulation to the $\ell_2$-loss soft-margin SVM (\ref{eq:softSVMprimal}) for $n$ datapoints $X_i\in\R^d$, together with their binary class labels $y_i\in \{\pm1\}$, for $i\in[1..n]$, as defined above in Section~\ref{sec:softSVM}.
The equivalence to~(\ref{eq:classifier}) directly extends to the one- and two-class case, without or with (regularized) offset term, and as well for the hard-margin SVM.
These equivalent formulations have been known in the SVM literature, see e.g. 
\cite{Scholkopf:2002ta,Keerthi:2000tj,Tsang:2005up,GartnerJaggi:2009}, 
and the references therein.

\begin{replemma}{lem:softSVM}
The dual of the soft-margin SVM (\ref{eq:softSVMprimal}) is an instance of the classifier formulation~(\ref{eq:classifier}), that is \ 
$
  \min_{x\in\Simplex} \, \norm{Ax}^2
$
, with
\[
A := {Z\choose\frac1{\sqrt{C}}\id_n} \in \R^{(d+n)\times n} \vspace{-1pt}
\]
where the data matrix $Z\in\R^{d\times n}$ consists of the $n$ columns $Z_i := y_iX_i$.
\end{replemma}
\begin{proof}
The Lagrangian \cite[Section 5]{Boyd:2004uz}\index{Lagrange duality} of the soft-margin SVM formulation (\ref{eq:softSVMprimal}) with its $n$ constraints can be written as 
\[
\begin{array}{rl}
L(w,\rho,\xi,\alpha) := & \frac12 \norm{w}^2 - \rho + \frac{C}{2}\sum_i \xi_i^2\\
   &+\sum_i \alpha_i \left( - w^T Z_i + \rho - \xi_i \right) \ .
\end{array}
\]
Here we introduced a non-negative Lagrange multiplier $\alpha_i\ge0$ for each of the $n$ constraints.
Differentiating $L$ with respect to the primal variables, we obtain the KKT optimality conditions\vspace{-1mm}
\[
\begin{array}{rcl}
   \zero \ {\buildrel ! \over =}& \frac{\partial}{\partial w} &=\ w - \sum_i \alpha_i Z_i \\
   0 \ {\buildrel ! \over =}& \frac{\partial}{\partial \rho} &=\ 1 - \sum_i \alpha_i \\
   \zero \ {\buildrel ! \over =}& \frac{\partial}{\partial \xi} &=\ C\xi - \alpha  \ .
\end{array}
\]
When plugged into the Lagrange dual problem $\max_{\alpha} \min_{w,\rho,\xi} L(w,\rho,\xi,\alpha)\,$, these give us the equivalent formulation (sometimes called the \emph{Wolfe-dual})
\[
\begin{array}{rl}
\max_{\alpha} & \frac12 \alpha^TZ^TZ\alpha - \rho + \frac{C}2\frac1{C^2}\alpha^T\alpha \\
 &~~~   -\alpha^TZ^TZ\alpha + \rho - \frac1C\alpha^T\alpha\\
 &= -\frac12\alpha^TZ^TZ\alpha - \frac1{2C}\alpha^T\alpha \ .
\end{array}
\]
In other words the dual\index{SVM dual optimization problem} is
\[
\begin{array}{rl}
\min_{\alpha} & \alpha^T\big(Z^TZ+ \frac1{C} \id_n\big)\alpha\\
 s.t.& \alpha \ge 0\\
     & \alpha^T\one = 1 \ .
\end{array}
\]
This is directly an instance of our first SVM formulation~(\ref{eq:classifier}) used in the introduction, if we use the extended matrix \vspace{-1mm}
\[
A := {Z\choose\frac1{\sqrt{C}}\id_n} \in \R^{(d+n)\times n} \ .  \vspace{-1em}
\]
\end{proof}

Note that the (unique) optimal primal solution $w$ can directly be obtained from any dual optimal $\alpha$ by using the optimality condition $w = A\alpha$.

%

\newpage
\bibliographystyle{alphaurl}
{\small
\bibliography{/Users/jaggim/bibliography}

\newcommand{\etalchar}[1]{$^{#1}$}
\begin{thebibliography}{WLG{\etalchar{+}}13}

\bibitem[BS83]{Bloomfield:1983tn}
Peter Bloomfield and William~L Steiger.
\newblock \href{http://books.google.ch/books?id=B1Q_AQAAIAAJ}{\em {Least absolute deviations: theory, applications, and
  algorithms}}.
\newblock Progress in probability and statistics. Birkh{\"a}user, 1983.

\bibitem[Bur98]{Burges:1998hg}
Christopher J~C Burges.
\newblock \href
  {http://dx.doi.org/10.1023/A:1009715923555}{A Tutorial on Support Vector Machines for Pattern Recognition}.
\newblock {\em Data Mining and Knowledge Discovery}, 2(2):121--167, 1998.

\bibitem[BV04]{Boyd:2004uz}
Stephen~P Boyd and Lieven Vandenberghe.
\newblock \href{http://www.stanford.edu/~boyd/cvxbook/}{\em {Convex optimization}}.
\newblock Cambridge University Press, 2004.

\bibitem[BvdG11]{Buhlmann:2011jp}
Peter B{\"u}hlmann and Sara van~de Geer.
\newblock \href{http://dx.doi.org/10.1007/978-3-642-20192-9}{\em {Statistics for High-Dimensional Data - Methods, Theory and
  Applications}}.
\newblock Springer Series in Statistics 0172-7397. Springer, 2011.

\bibitem[CBSSS11]{CesaBianchi:2011wb}
Nicol{\`o} Cesa-Bianchi, Shai Shalev-Shwartz, and Ohad Shamir.
\newblock \href{http://jmlr.org/papers/v12/cesa-bianchi11a.html}{Efficient Learning with Partially Observed Attributes}.
\newblock {\em The Journal of Machine Learning Research}, 12:2857--2878,
  2011.

\bibitem[CDS98]{Chen:1998hm}
Scott~Shaobing Chen, David~L Donoho, and Michael~A Saunders.
\newblock \href{http://dx.doi.org/10.1137/S1064827596304010}{Atomic Decomposition by Basis Pursuit}.
\newblock {\em SIAM Journal on Scientific Computing}, 20(1):33, 1998.

\bibitem[CHL08]{Chang:2008up}
Kai-Wei Chang, Cho-Jui Hsieh, and Chih-Jen Lin.
\newblock \href{http://jmlr.org/papers/v9/chang08a.html}{Coordinate Descent Method for Large-scale L2-loss Linear Support
  Vector Machines}.
\newblock {\em JMLR}, 9:1369--1398, 2008.

\bibitem[CHW10]{Clarkson:2010ue}
Kenneth~L Clarkson, Elad Hazan, and David~P Woodruff.
\newblock \href{http://dx.doi.org/10.1109/FOCS.2010.50}{Sublinear Optimization for Machine Learning}.
\newblock {\em FOCS 2010 - 51st Annual IEEE Symposium on Foundations of
  Computer Science}, 2010.

\bibitem[CV95]{cortes95softmargin}
Corinna Cortes and Vladimir Vapnik.
\newblock \href{http://dx.doi.org/10.1007/BF00994018}{Support-Vector Networks}.
\newblock {\em Machine Learning}, 20(3):273--297, 1995.

\bibitem[EHJT04]{Efron:2004tz}
Bradley Efron, Trevor Hastie, Iain Johnstone, and Robert Tibshirani.
\newblock \href{http://www.jstor.org/stable/3448465}{Least angle regression}.
\newblock {\em Annals of Statistics}, 32(2):407--499, 2004.

\bibitem[EPP00]{Evgeniou:2000fx}
Theodoros Evgeniou, Massimiliano Pontil, and Tomaso Poggio.
\newblock \href{http://dx.doi.org/10.1023/A:1018946025316}{Regularization Networks and Support Vector Machines}.
\newblock {\em Advances in Computational Mathematics}, 13(1):1--50, 2000.

\bibitem[GC05]{Ghosh:2005ft}
Debashis Ghosh and Arul~M Chinnaiyan.
\newblock \href{http://dx.doi.org/10.1155/JBB.2005.147}{Classification and Selection of Biomarkers in Genomic Data Using
  LASSO}.
\newblock {\em Journal of Biomedicine and Biotechnology}, 2005(2):147--154,
  2005.

\bibitem[Gir98]{Girosi:1998dd}
Federico Girosi.
\newblock \href{http://dx.doi.org/10.1162/089976698300017269}{An Equivalence Between Sparse Approximation and Support Vector
  Machines}.
\newblock {\em Neural Computation}, 10(6):1455--1480, 1998.

\bibitem[GJ09]{GartnerJaggi:2009}
Bernd G{\"a}rtner and Martin Jaggi.
\newblock \href{http://dx.doi.org/10.1145/1542362.1542370}{Coresets for polytope distance}.
\newblock {\em SCG '09: Proceedings of the 25th annual Symposium on
  Computational Geometry}, 2009.

\bibitem[GJL12]{Giesen:2012ik}
Joachim Giesen, Martin Jaggi, and S{\"o}ren Laue.
\newblock \href{http://dx.doi.org/10.1145/2390176.2390186}{Approximating parameterized convex optimization problems}.
\newblock {\em ACM Transactions on Algorithms}, 9(10):1--17, 2012.

\bibitem[GJM12]{Gartner:2012wi}
Bernd G{\"a}rtner, Martin Jaggi, and Cl{\'e}ment Maria.
\newblock \href{http://jocg.org/index.php/jocg/article/view/88/34}{An Exponential Lower Bound On The Complexity Of Regularization
  Paths}.
\newblock {\em Journal of Computational Geometry}, 3(1):168--195, 2012.

\bibitem[GMLS12]{Giesen:2012uk}
Joachim Giesen, Jens M{\"u}ller, Soeren Laue, and Sascha Swiercy.
\newblock \href{http://books.nips.cc/papers/files/nips25/NIPS2012_1034.pdf}{Approximating Concavely Parameterized Optimization Problems}.
\newblock In {\em NIPS}, 2012.

\bibitem[GVR10]{Ghaoui:2010tu}
Laurent~El Ghaoui, Vivian Viallon, and Tarek Rabbani.
\newblock \href{http://arxiv.org/abs/1009.4219v2}{Safe Feature Elimination for the LASSO and Sparse Supervised
  Learning Problems}.
\newblock {\em arXiv.org}, 2010.

\bibitem[Hea82]{Hearn:1982ee}
Donald~W Hearn.
\newblock \href{http://dx.doi.org/10.1016/0167-6377(82)90049-9}{The gap function of a convex program}.
\newblock {\em Operations Research Letters}, 1(2):67--71, 1982.

\bibitem[HK12]{Hazan:2012wj}
Elad Hazan and Tomer Koren.
\newblock \href{http://arxiv.org/abs/1108.4559}{Linear Regression with Limited Observation}.
\newblock In {\em ICML 2012 - Proceedings of the 29th International Conference
  on Machine Learning}, 2012.

\bibitem[HKS11]{Hazan:2011wq}
Elad Hazan, Tomer Koren, and Nathan Srebro.
\newblock \href{http://ie.technion.ac.il/~ehazan/papers/HaKoSr11full.pdf}{Beating SGD: Learning SVMs in Sublinear Time}.
\newblock In {\em NIPS}, 2011.

\bibitem[HO04]{Hochreiter:2004wv}
Sepp Hochreiter and Klaus Obermayer.
\newblock \href{http://books.google.com/books?hl=en&lr=&id=SwAooknaMXgC&oi=fnd&pg=PA319&dq=Gene+Selection+for+Microarray+Data+Hochreiter&ots=rJz9BqRdAi&sig=f7L66mKzJBfQKB_AfbNDggXfa7c}
{Gene Selection for Microarray Data}.
\newblock In Bernhard Sch{\"o}lkopf, Jean-Philippe Vert, and Koji Tsuda,
  editors, {\em Kernel Methods in Computational Biology}, page 319. MIT Press,
  2004.
  
\bibitem[HO06]{Hochreiter:2006jc}
Sepp Hochreiter and Klaus Obermayer.
\newblock \href{http://dx.doi.org/10.1162/neco.2006.18.6.1472}{Support Vector Machines for Dyadic Data}.
\newblock {\em Neural Computation}, 18(6):1472--1510, 2006.

\bibitem[HRTZ04]{Hastie:2004uj}
Trevor Hastie, Saharon Rosset, Robert Tibshirani, and Ji~Zhu.
\newblock \href{http://jmlr.org/papers/v5/hastie04a.html}{The Entire Regularization Path for the Support Vector Machine}.
\newblock {\em The Journal of Machine Learning Research}, 5:1391--1415,
  2004.

\bibitem[Jag13a]{Jaggi:2013vta}
Martin Jaggi.
\newblock \href{http://arxiv.org/pdf/1303.1152v1.pdf}{An Equivalence between the Lasso and Support Vector Machines}.
\newblock In {\em ROKS 2013: International Workshop on Advances in
  Regularization, Optimization, Kernel Methods and Support Vector Machines:
  Theory and Applications}, 2013.

\bibitem[Jag13b]{Jaggi:2013wg}
Martin Jaggi.
\newblock \href{http://jmlr.org/proceedings/papers/v28/jaggi13}{Revisiting Frank-Wolfe: Projection-Free Sparse Convex Optimization}.
\newblock In {\em ICML 2013 - Proceedings of the 30th International Conference
  on Machine Learning}, 2013.

\bibitem[KSBM00]{Keerthi:2000tj}
S~Sathiya Keerthi, Shirish~K Shevade, Chiranjib Bhattacharyya, and K~R~K
  Murthy.
\newblock \href{http://dx.doi.org/10.1109/72.822516}{A fast iterative nearest point algorithm for support vector machine
  classifier design}.
\newblock {\em IEEE Transactions on Neural Networks}, 11(1):124--136, 2000.

\bibitem[LM01]{Lee:2001vf}
Yuh-Jye Lee and Olvi~L Mangasarian.
\newblock \href{http://dx.doi.org/10.1137/1.9781611972719.13}{RSVM: Reduced Support Vector Machines}.
\newblock In {\em SDM 2001 - Proceedings of the first SIAM International
  Conference on Data Mining}, 2001.

\bibitem[LYX05]{Li:2005wr}
Fan Li, Yiming Yang, and Eric~P Xing.
\newblock \href{http://books.nips.cc/papers/files/nips18/NIPS2005_0644.pdf}{From Lasso regression to Feature vector machine}.
\newblock In {\em NIPS}, 2005.

\bibitem[MY12]{Mairal:2012wu}
Julien Mairal and Bin Yu.
\newblock \href{http://arxiv.org/abs/1205.0079}{Complexity Analysis of the Lasso Regularization Path}.
\newblock In {\em ICML 2012 - Proceedings of the 29th International Conference
  on Machine Learning}, 2012.

\bibitem[OPT00]{Osborne:2000kg}
Michael~R Osborne, Brett Presnell, and Berwin~A Turlach.
\newblock \href{http://dx.doi.org/10.1093/imanum/20.3.389}{A new approach to variable selection in least squares problems}.
\newblock {\em IMA Journal of Numerical Analysis}, 20(3):389--403, 2000.

\bibitem[OST13]{Ogawa:2013ul}
Kohei Ogawa, Yoshiki Suzuki, and Ichiro Takeuchi.
\newblock \href{http://jmlr.org/proceedings/papers/v28/ogawa13b.html}{Safe Screening of Non-Support Vectors in Pathwise SVM Computation}.
\newblock In {\em ICML 2013 - Proceedings of the 30th International Conference
  on Machine Learning}, pages 1382--1390, 2013.

\bibitem[PRE98]{Pontil:1998vs}
Massimiliano Pontil, Ryan Rifkin, and Theodoros Evgeniou.
\newblock \href{http://hdl.handle.net/1721.1/7258}{From Regression to Classification in Support Vector Machines}.
\newblock Technical Report A.I. Memo No. 1649, MIT, 1998.

\bibitem[PS12]{Porat:2012ud}
Ely Porat and Martin~J Strauss.
\newblock \href{http://dx.doi.org/10.1137/1.9781611973099.96}{Sublinear time, measurement-optimal, sparse recovery for all}.
\newblock In {\em SODA '12: Proceedings of the Twenty-Third Annual ACM-SIAM
  Symposium on Discrete Algorithms}. SIAM, 2012.

\bibitem[Ros58]{Rosenblatt:1958}
Frank Rosenblatt.
\newblock \href{http://dx.doi.org/10.1037/h0042519}{The Perceptron: A Probabilistic Model for Information Storage and
  Organization in the Brain}.
\newblock {\em Psychological Review}, 65(6):386--408, 1958.

\bibitem[Rot04]{Roth:2004gw}
Volker Roth.
\newblock \href{http://dx.doi.org/10.1109/TNN.2003.809398}{The Generalized LASSO}.
\newblock {\em IEEE Transactions on Neural Networks}, 15(1):16--28,
  2004.

\bibitem[SGV98]{Saunders:1998ws}
Craig Saunders, Alexander Gammerman, and Volodya Vovk.
\newblock \href{http://portal.acm.org/citation.cfm?id=645527.657464&coll=DL&dl=GUIDE&CFID=67533571&CFTOKEN=24674007}
 {Ridge Regression Learning Algorithm in Dual Variables}.
\newblock In {\em ICML '98: Proceedings of the Fifteenth International
  Conference on Machine Learning}. 1998.
  
\bibitem[SS02]{Scholkopf:2002ta}
Bernhard Sch{\"o}lkopf and Alex~J Smola.
\newblock \href{http://books.google.com/books?id=y8ORL3DWt4sC&printsec=frontcover&dq=kernel+learning&cd=1&source=gbs_api}
  {\em {Learning with kernels}}.
\newblock support vector machines, regularization, optimization, and beyond.
  The MIT Press, 2002.
  
\bibitem[SS04]{Smola:2004ba}
Alex~J Smola and Bernhard Sch{\"o}lkopf.
\newblock \href{http://dx.doi.org/10.1023/B:STCO.0000035301.49549.88}{A tutorial on support vector regression}.
\newblock {\em Statistics and Computing}, 14:199--222, 2004.

\bibitem[Ste03]{Steinwart:2003wk}
Ingo Steinwart.
\newblock \href{https://papers.nips.cc/paper/2477-sparseness-of-support-vector-machines-some-asymptotically-sharp-bounds.pdf}
  {Sparseness of Support Vector Machines---Some Asymptotically Sharp
  Bounds}.
\newblock In {\em NIPS}, 2003.

\bibitem[TBF{\etalchar{+}}11]{Tibshirani:2011iz}
Robert Tibshirani, Jacob Bien, Jerome Friedman, Trevor Hastie, Noah Simon,
  Jonathan Taylor, and Ryan~J. Tibshirani.
\newblock \href {http://dx.doi.org/10.1111/j.1467-9868.2011.01004.x}{Strong rules for discarding predictors in lasso-type problems}.
\newblock {\em Journal of the Royal Statistical Society: Series B (Statistical
  Methodology)}, 74(2):245--266, 2011.

\bibitem[Tib96]{Tibshirani:1996wb}
Robert Tibshirani.
\newblock \href{http://www.jstor.org/stable/2346178}{Regression Shrinkage and Selection via the Lasso}.
\newblock {\em Journal of the Royal Statistical Society. Series B
  (Methodological)}, pages 267--288, 1996.

\bibitem[TKC05]{Tsang:2005up}
Ivor~W Tsang, James~T Kwok, and Pak-Ming Cheung.
\newblock \href{http://jmlr.org/papers/volume6/tsang05a/tsang05a.pdf}{Core Vector Machines: Fast SVM Training on Very Large Data Sets}.
\newblock {\em Journal of Machine Learning Research}, 6:363--392, 2005.

\bibitem[TRS12]{Thiagarajan:2012tb}
Jayaraman~J Thiagarajan, Karthikeyan~Natesan Ramamurthy, and Andreas Spanias.
\newblock \href{http://www.public.asu.edu/~knatesan/papers/Local_Sparse_Coding_Classification.pdf}
  {Local Sparse Coding for Image Classification and Retrieval}.
\newblock Technical report, ASU, 2012.

\bibitem[WLG{\etalchar{+}}13]{Wang:2013uq}
Jie Wang, Binbin Lin, Pinghua Gong, Peter Wonka, and Jieping Ye.
\newblock \href{http://arxiv.org/abs/1211.3966}{Lasso Screening Rules via Dual Polytope Projection}.
\newblock In {\em NIPS}, 2013.

\bibitem[Zie95]{Ziegler:1995td}
G{\"u}nter~M Ziegler.
\newblock \href{http://www.springer.com/mathematics/geometry/book/978-0-387-94329-9#}{\em {Lectures on Polytopes}}, volume 152 of {\em Graduate Texts in
  Mathematics}.
\newblock Springer Verlag, 1995.

\end{thebibliography}
}

\end{document}